\documentclass{article}
\usepackage[hidelinks,colorlinks,linkcolor=black,citecolor=orange]{hyperref}

\usepackage[accepted]{aux/mlsys2025}
\usepackage{multirow}
\usepackage{subcaption}
\usepackage{tabularx}
\usepackage{aux/EP_macros}
\usepackage{multirow}

\usepackage{hyperref}

\usepackage[perpage,symbol*]{footmisc}
\DefineFNsymbols*{namedisp}{
{*}
{**}
{$\dagger$}                                     {$\ddagger$}\S%
                            \P{$\Vert$}{$\dagger\dagger$}%
                            {$\ddagger\ddagger$}}
\setfnsymbol{namedisp}

\usepackage{makecell} 

\usepackage{amsmath,amsfonts,bm}

\def\eqref#1{equation~\ref{#1}}

\def\1{\bm{1}}

\DeclareMathAlphabet{\mathsfit}{\encodingdefault}{\sfdefault}{m}{sl}
\SetMathAlphabet{\mathsfit}{bold}{\encodingdefault}{\sfdefault}{bx}{n}

\usepackage{url}

\usepackage{listings}
\usepackage{xcolor}

\def\tmp{{\sf tmp}}

\usepackage{booktabs}
\mlsystitlerunning{Fast training of large kernel models}

\begin{document}

\twocolumn[
\mlsystitle{ 
Fast training of large kernel models with delayed projections}




\mlsyssetsymbol{equal}{*}

\begin{mlsysauthorlist}
\mlsysauthor{Amirhesam Abedsoltan}{sd}
\mlsysauthor{Siyuan Ma}{goo}
\mlsysauthor{Parthe Pandit}{iit}
\mlsysauthor{Mikhail Belkin}{sd}
\end{mlsysauthorlist}

\mlsysaffiliation{sd}{UC San Diego}
\mlsysaffiliation{goo}{Google}
\mlsysaffiliation{iit}{IIT Bombay}

\mlsyscorrespondingauthor{Amirhesam Abedsoltan}{aabedsoltan@ucsd.edu}
\mlsyscorrespondingauthor{Siyuan Ma}{siyuan.ma9@gmail.com}
\mlsyscorrespondingauthor{Parthe Pandit}{pandit@iitb.ac.in}

\mlsyskeywords{Kernel Machines, Kernel methods, Inducing points, Sparse Gaussian Processes}

\vskip 0.3in

\begin{abstract}
\vskip 0.3in

Classical kernel machines have historically faced significant challenges in scaling to large datasets and model sizes—a key ingredient that has driven the success of neural networks. In this paper, we present a new methodology for building kernel machines that can scale efficiently with both data size and model size. Our algorithm introduces delayed projections to Preconditioned Stochastic Gradient Descent (PSGD) allowing the training of much larger models than was previously feasible, pushing the practical limits of kernel-based learning.
 We validate our algorithm, \EP4, across multiple datasets, demonstrating drastic training  speed up over the existing methods while maintaining comparable or better classification accuracy.
 
\end{abstract}
]

\printAffiliations{}


\section{Introduction}



Kernel methods have strong theoretical foundations and broad applicability. They have also served as the foundation for understanding many significant phenomena in modern machine learning~\citep{jacot2018neural,belkin2018understand,belkin2019reconciling,zhang2021understanding}. Despite these advantages, the scalability of kernel methods has remained a persistent challenge, particularly when applied to large datasets. Addressing this limitation is critical for expanding the utility of kernel-based techniques in modern machine learning applications.


A naive approach for training kernel machines is to directly solve the equivalent kernel matrix inversion problem. In general, the computational complexity of solving the kernel matrix inversion problem is $O(n^3)$, where $n$ is the number of training samples. Thus, computational cost grows rapidly with the size of the dataset, making it computationally intractable for datasets with more than $\sim10^5$ data points.

To address this challenge, various methods employing iterative algorithms and approximations have been proposed. Among these, Gradient Descent (GD)-based algorithms like Pegasos~\citep{shalev2007pegasos} and \EP{1.0,2.0}~\citep{ma2017diving,ma2019kernel} have significantly reduced the computational complexity to $O(n^2)$. These methods, adaptable for stochastic settings, offer more efficient implementations. Nevertheless, the scalability of kernel machines remains constrained by the inherent linkage between the model size and the training set.

Furthermore, the Nystr\"om methods have emerged as a favored approach for scaling kernel machines, with seminal works with~\citep{williams2000using} paving the way. Methods such as {\NYTRO}~\citep{camoriano2016nytro}, {\FALKON}~\citep{rudi2017falkon} and \askotch~\citep{askotch} leverage the Nystr\"om Approximation (NA) in combination with other strategies to enhance performance. {\NYTRO} merges NA with gradient descent to improve condition number, \askotch~combines it with block coordinate descent, whereas \FALKON~combines it with the Conjugate Gradient method, facilitating the handling of large training sets. However, these strategies are limited by model size due to memory restrictions, exhibiting quadratic scaling in relation to the size of the model. For instance, scaling \FALKON~\citep{meanti2020kernel} method to a model size of $512,000$ necessitates over 1TB of RAM, surpassing the capacity of most high-end servers available today. 

Other lines of work in the Gaussian Processes literature, e.g., \citep{titsias2009variational,wilson2015kernel,gardner2018product,GPflow2017}, use so-called \textit{inducing points} to control model complexity. However, these methods face similar scaling issues as they require quadratic memory in terms of the number of inducing points, thus preventing scaling to large models.

Recently, \EP{3.0} was introduced in~\citep{abedsoltan_2023}. Unlike previous versions, \EP{3.0} distangle the model from the training set, similar to \FALKON, but with the added advantage that its memory requirements scale linearly with the model size. This advancement makes it feasible to tackle kernel models of sizes previously deemed unattainable. However, its per iteration time complexity remains quadratic relative to the model size, significantly slowing its practical application. 


In this paper, we build upon \EP{3.0} and introduce \EP{4.0}\footnote{\url{https://github.com/EigenPro/EigenPro/tree/main}}. This new algorithm retains the advantageous features of \EP{3.0}, such as decoupling the model from the training set and linear scaling in memory complexity. Moreover, it significantly improves upon the time complexity, achieving amortized linear scaling per iteration with respect to model size. We empirically verify that the proposed algorithm converges in fewer epochs, without compromising generalization performance.

\subsection{Main contribution}

Our method for kernel machine problems achieves three key advantages: (1) linear amortized time complexity per iteration, (2) linear memory scaling with model size, (3) comparable or superior performance compare to existing methods while achieving up to 600× speedup in our experiments, and (4) an empirically significant reduction in the number of epochs required for convergence, particularly for larger models.  Figure~\ref{fig:time_comparision} demonstrates these benefits on CIFAR5M data set.

 \begin{figure}[h]
    \centering
\includegraphics[width=0.5\textwidth, trim=0 0 -10 0, clip]{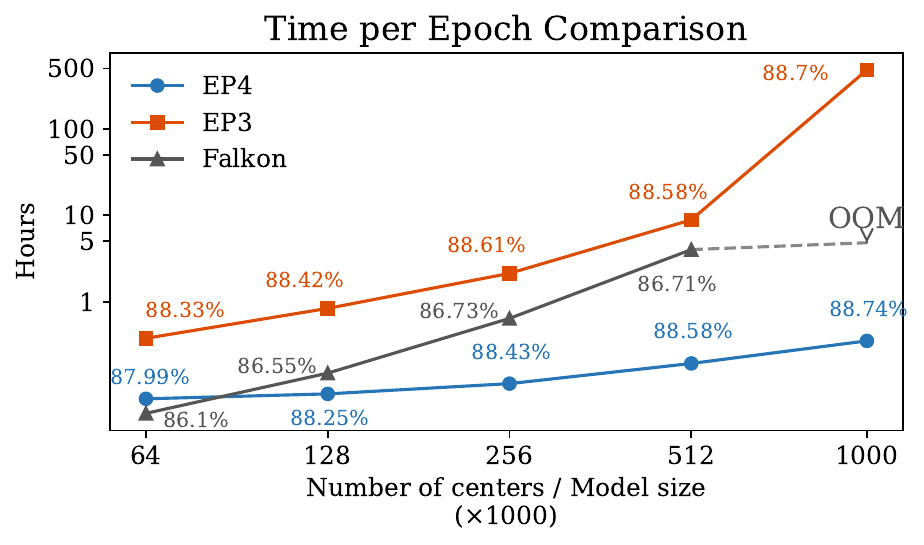}
    \caption{Per epoch time comparison between different solvers. Performance in terms of classification test accuracy (indicated as percentages) is annotated next to each data point, showing that EP4 maintains superior or comparable performance across all model sizes. The detail of the experiment can be found in \Cref{appendix:expts}.}
    \label{fig:time_comparision}
\end{figure}

\subsection{Organization of the Paper}
The remainder of this paper is organized as follows.  Section \ref{sec:prelim} provides the necessary background and preliminaries. In Section \ref{sec:highlevel}, we present a high-level overview of \EP4 and introduce the key insights that enable its dramatic improvement in computational efficiency. In Section \ref{sec:derivation}, we derive the complete algorithm and present our computational optimization techniques. Finally, Section \ref{sec:expt} presents extensive experimental results across multiple datasets and model sizes.

\section{Notation and Background}
\label{sec:prelim}

In what follows, functions are lowercase letters $a$, sets are uppercase letters $A$, vectors are lowercase bold letters $\a$, matrices are uppercase bold letters $\A$, operators are calligraphic letters $\mathcal{A},$ spaces and sub-spaces are boldface calligraphic letters $\bm{\mathcal{A}}.$ 

\textbf{General kernel models.} Following \EP{3.0} \citep{abedsoltan_2023} notations, given training data $(X,\y)=\curly{\x_i\in \Real^d,y_i\in \Real}_{i=1}^n$, \textit{General kernel models} are models of the form,
\begin{align*}
f(\x) = \sum_{i=1}^p\alpha_i K(\x,\z_i).
\end{align*}

Here, $K:\mathbb{R}^d\times \mathbb{R}^d\rightarrow\mathbb{R}$ is a positive semi-definite symmetric kernel function 
and $Z=\{z_i\in\mathbb{R}^d\}_{i=1}^p$ is the set of \textit{centers}, which is not necessary the same as the training set. We will refer to $p$ as the \textit{model size}. We further define $\Hilbert$ is the (unique) reproducing kernel Hilbert space (RKHS) corresponding to $K$.

\textbf{Loss function.} Our goal will be to find the solution to the following infinite-dimensional Mean Squared Error (MSE) problem for general kernel models,
\begin{align}\label{eq:loss_min}
    &\minimize{f\in \Hilbert}\, L(f)=\sum_{i=1}^n (f(\x_i)-y_i)^2,\\
    &\subjectto\quad {f\in\Zspan}:=\text{span}\!\round{\curly{K(\cdot,\z_j)}_{j=1}^p}.
\end{align}

\textbf{Evaluations and kernel matrices.} The vector of evaluations of a function $f$ over a set $X=\curly{\x_i}_{i=1}^n$ is denoted $f(X):=(f(\x_i))\in\Real^n$. For sets $X$ and $Z$, with $|X|=n$ and $|Z|=p$, we denote the kernel matrix $K(X,Z)\in\Real^{n\times p},$ while $K(Z,X)=K(X,Z)\T$.
Similarly, $K(\cdot,X)\in\Hilbert^{n}$ is a vector of functions, and we use $K(\cdot,X)\alphavec:=\sum_{i=1}^n K(\cdot,\x_i)\alpha_i\in\Hilbert,$ to denote their linear combination.
Finally, for an operator $\mc A,$ a function $a$, and a set $A=\{\a_i\}_{i=1}^k$, we denote the vector of evaluations of the output, 
\begin{align}
 \mc A\curly{a}(A)   := (b(\a_i))\in\Real^k\qquad \text{where}\quad b=\mc A\round{a}.
\end{align}

\textbf{Fr\'echet derivative.} Given a function $J:\Hilbert\to \Real$,
the Fr\'echet derivative of $J$ with respect to $f$ is a linear functional, denoted $\nabla_f J$, such that for $h\in\Hilbert$
\begin{align}
    \lim_{\norm{h}_\Hilbert\rightarrow 0}\frac{\abs{J(f+h)-J(f)-\nabla_f J(h)}}{\norm{h}_\Hilbert}=0.
\end{align}
Since $\nabla_f J$ is a linear functional, it lies in the dual space $\Hilbert^*.$ Since $\Hilbert$ is a Hilbert space, it is self-dual, whereby $\Hilbert^*=\Hilbert.$
If $f$ is a general kernel model, and $L$ is the square loss for a given dataset $(X,\y)$, \ie, $L(f):=\frac12\sum_{i=1}^n(f(\x_i)-y_i)^2$ we can apply the chain rule, and using reproducing property of $\Hilbert$, and the fact that $\nabla_f\inner{f,g}_\Hilbert=g$, we get, that the Fr\'echet derivative of $L$, at $f=f_0$ is, 
\begin{align}\label{eq:frechet_square_loss}
    \nabla_f L(f_0)&=\sum_{i=1}^n(f_0(\x_i)-y_i)\nabla_{\!f} f(\x_i) \\
    &= K(\cdot,X)(f_0(X)-\y).
\end{align}

\textbf{Hessian operator.} The Hessian operator $\nabla^2_f L:\Hilbert\rightarrow\Hilbert$ for the square loss is given by,
\begin{align}\label{eq:hessian}
    &\mc K:= \sum_{i=1}^n K(\cdot,\x_i)\otimes K(\cdot,\x_i),\\ 
    &\mc K\curly{f}(\z)\!=\!
    \sum_{i=1}^n K(\z,\x_i)f(\x_i)=K(\z,X)f(X).
\end{align}

Operator $\mc K$ has non-negative eigenvalues which we assume are ordered as $\lambda_1\geq \lambda_2 \geq \dotsb\geq\lambda_n\geq 0$. Hence we have eigen-decompositions for this operator written as $\mc K = \sum_{i=1}^{n} \lambda_i\cdot \psi_i\otimes\psi_i$. Combining Equations \ref{eq:frechet_square_loss} and \ref{eq:hessian}, we can rewrite the Fréchet derivative of the loss function as following:

\begin{equation}\label{eq:derivative_hessian}
\nabla_f L(f_0)(z) = \mc K\curly{f_0(X)-\y}(z) 
\end{equation}

\textbf{Exact minimum norm solution.} The closed-form minimum $\norm{\cdot}_{\Hilbert}$ norm solution to the problem defined in \cref{eq:loss_min} is given by:

\begin{equation}\label{eq:pinvers_solution}
    \hat{f} := K(\cdot, Z) K\pinv(Z, X) \y,
\end{equation}

where $+$ is pseudoinverse or Moore–Penrose inverse. In the case of $X=Z$ it simplifies to $\hat{f} := K(\cdot, X) K\inv(X, X) \y$.

\textbf{Gradient Descent (GD).} If we apply GD on the optimization problem in \ref{eq:loss_min}, with learning rate $\eta$, in the $\Hilbert$ functional space, the update is as following,

\begin{subequations}
\begin{align}\label{eq:sgd}
    f_{t+1} &= f_{t} - \eta \cdot \nabla_f L(f_t) \\
    &=f_{t} - \eta  K(\cdot,X)\round{f_t(X)-\y}.
\end{align}
\end{subequations}
The first point to note is that the derivative lies in $\Xspan := \text{span}\round{\curly{K(\cdot,\x_j)}_{j=1}^n}$ rather than in $\Zspan$. Therefore, when $X \neq Z$, SGD cannot be applied in this form. We will revisit this issue later. The second point is that in the case of $X = Z$, traditional \textit{kernel regression} problem, the convergence of SGD depends on the condition number of $\mc K$. Simply put, this is the ratio of the largest to the smallest non-zero singular value of the Hessian operator defined in \ref{eq:hessian}. It is known that for general kernel models this condition number is usually ill conditioned and converges slow, see \citep{abedsoltan2023nystrom} for more details on this.

\begin{figure*}[h!]
    \centering
    \includegraphics[clip, trim=6cm 9cm 6cm 9cm, width=0.85\textwidth]{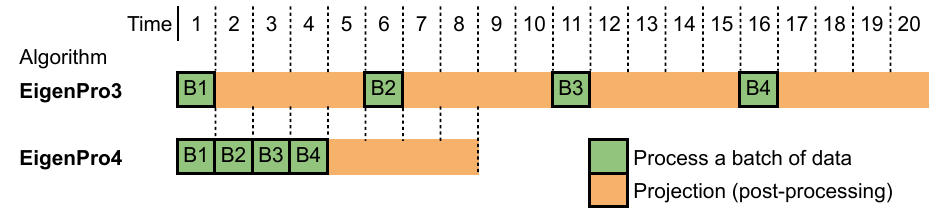}
    \caption{\textbf{Design of EigenPro4.} An illustration of how batches of data are processed by the two algorithms. EigenPro3 involves an expensive  \textit{projection} step when processing every batch of data. EigenPro4 waits for multiple batches to be processed before running the projection step for all of them together. This reduces the amortized cost for processing each batch.}
    \label{fig:illustration_ep4_vs_ep3}
\end{figure*}

\textbf{\EP.} Prior work, \EP, by \citep{ma2017diving}, addresses the slow convergence of SGD by introducing a preconditioned stochastic gradient descent mechanism in Hilbert spaces. The update rule is the same as Equation \ref{eq
} but with an additional preconditioner $\mc{P}:\Hilbert\rightarrow\Hilbert$ applied to the gradient,

\begin{align}\label{eq
} f_{t+1} = f_{t} - \eta \cdot \mathcal{P}{\nabla_f L(f_t)}. \end{align}

In short, the role of the preconditioner $\mc{P}$ is to suppress the top eigenvalues of the Hessian operator $\mc K$ to improve the condition number. We next explicitly define $\mc P$.

\begin{definition}[Top-$q$ Eigensystem]\label{def
} Let $\lambda_1 > \lambda_2 > \ldots > \lambda_n$ be the eigenvalues of a Hermitian matrix $\A\in\Real^{n\times n}$, where for unit-norm $\e_i$, we have $\A\e_i = \lambda_i\e_i$. We define the tuple $(\Lambda_q, \E_q, \lambda_{q+1})$ as the top-$q$ eigensystem, where: 
\begin{subequations}
\begin{align} 
    &\Lambda_q := {\rm diag}(\lambda_1, \lambda_2, \ldots, \lambda_q) \in \Real^{q \times q}, \\ 
    &\E_q := [\e_1, \e_2, \ldots, \e_q] \in \Real^{n \times q}. \end{align} 
\end{subequations}
\end{definition}

If $\A=K(X_s,X_s)$, we also define the following objects that is used in the iterations of \EP4
\begin{subequations}\label{eq:def_D_and_F}    
\begin{align}
    \D &:= \Lambda\inv - \lambda_{q+1}\Lambda^{-2}\\
    \F &:= \E\sqrt{\D}
\end{align}
\end{subequations}

\textbf{Preconditioner.} Using Definition \ref{def
}, let $(\Lambda_q, \E_q, \lambda_{q+1})$ as the top-$q$ eigensystem of $K(X,X)$, the preconditioner $\mc{P}:\Hilbert\to\Hilbert$ can be explicitly written as following,

\begin{align}
\label{eq:exact_preconditioner}
\mc P := \mathcal{I}  - \sum_{i=1}^q \round{1 - \frac{\lambda_{q+1}}{\lambda_{q}}} \psi_i \otimes \psi_i. 
\end{align}

\textbf{Nystr\"om approximate preconditioner.} \EP2, introduced by \citep{ma2019kernel}, implements a stochastic approximation for $\mc{P}$ based on the Nyström extension, thereby reducing the time and memory complexity compared to \EP. The first step is to approximate the Hessian operator using the Nyström extension as follows,

\begin{subequations}
\label{eq:ps} 
\begin{align}
\mc K^s &:= \sum_{k=1}^s K(\cdot,\x_{i_k})\otimes K(\cdot,\x_{i_k}) \\
&= \sum_{i=1}^{s} \lambda_i^s \cdot \psi_i^s \otimes \psi_i^s. 
\end{align}
\end{subequations}

This is a Nyström approximation of \(\mc K\) using \(s\) uniformly random samples from \(X\), referred to as \(X_s\), where \((\Lambda_q^s, \E_q^s, \lambda_{q+1}^s)\) represents the corresponding top-\(q\) eigensystem of \(K(X_s, X_s)\). Using this approximation, we can define the approximated preconditioner as follows,
\begin{align} 
\label{eq:nystrom_preconditioner}
\mc P^s := \mathcal{I} - \sum_{i=1}^q \round{1 - \frac{\lambda_{q+1}^s}{\lambda_{q}^s}} \psi_i^s \otimes \psi_i^s 
\end{align}
For more details on the performance of this preconditioner compared to the case of $s=n$, see \citet{abedsoltan2023nystrom}, who showed that choosing $s\gtrsim\log^4n$ is sufficient.

Of particular importance is the action of this preconditioner on any function of the form $K(\cdot,A)\bm{u}$.
\begin{subequations}
\label{eq:action_of_nystrom_preconditioner}
\begin{align}
    &\mc P^s K(\cdot, A)\bm{u}
    =K(\cdot, A)\bm{u}-\sum_{i=1}^q \round{1 - \tfrac{\lambda_{q+1}^s}{\lambda_{q}^s}}\psi_i \psi_i(A)\tran\bm{u}\\
    &= K(\cdot,A)\bm{u}\nonumber\\
    &\quad-\sum_{i=1}^q \round{1 - \tfrac{\lambda_{q+1}^s}{\lambda_{q}^s}}\frac{K(\cdot,X_s)\e_i}{\sqrt{\lambda_i}} \frac{\e_i\tran K(X_s,A)}{\sqrt{\lambda_i}} \bm{u}\\
    &=K(\cdot,A)\bm{u}-K(\cdot,X_s)\E\D\E\tran K(X_s,A)\bm{u}\\
    &=K(\cdot,A)\bm{u}-K(\cdot,X_s)\F\F\tran K(X_s,A)\bm{u}
\end{align}
\end{subequations}
where recall the definitions of $\D$ and $\F$ in \cref{eq:def_D_and_F}.



\textbf{\EP3.} The primary limitation of \EP2 was its inability to handle cases where \(Z \neq X\), a necessary condition for disentangling the model and the training set. \EP3 overcomes this limitation by recognizing that although the gradients in \eqref{eq:frechet_square_loss} may not lie within \(\Zspan\), it is possible to project them back to \(\Zspan\). Consequently, \EP3 can be summarized as follows:
\begin{align}\label{eq:ep3}
    f_{t+1} =  \text{proj}_{\Zspan}\left(f_t - \eta \mathcal{P}^s\{\wt\nabla_f L(f_t)\}\right),
\end{align}
where $\text{proj}_{\Zspan}\round{u} := \argmin{f\in\Zspan}\norm{u-f}_\Hilbert^2$ for any $u \in \Hilbert$. As shown in \citep[Section 4.2 ]{abedsoltan_2023}, the exact projection is,

\begin{align}\label{eq:projection}
    \text{proj}_{\Zspan}(u) = K(\cdot,Z)K^{-1}(Z,Z) u(Z).
\end{align}

This projection can be interpreted as solving a kernel in $\Zspan$ and can be approximated using \EP2, as done in \citep{abedsoltan_2023}, with a time complexity that scales quadratically with model size. However, since this projection must be performed after each stochastic step, it becomes the most computationally expensive part of the \EP3 algorithm.

\begin{figure*}[!ht]
    \centering
\includegraphics[width=1.0\textwidth]{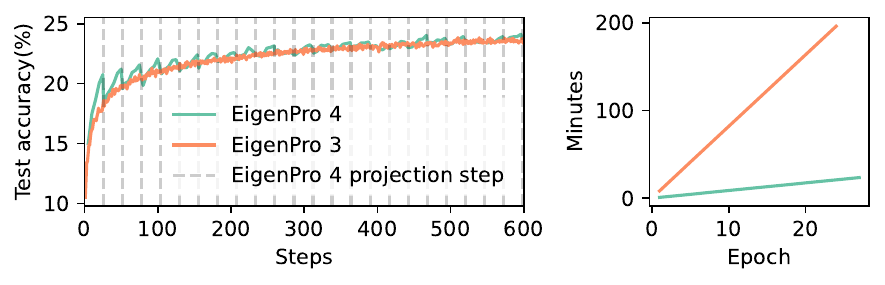}
\vspace{-3pt}
    \caption{Performance and computational time comparison between \EP{4.0} ($T=11$) and \EP{3.0} (equivalent to $T=1$), highlighting the impact of the projection step on the performance of \EP{4.0}. The detail of the experiment can be found in \Cref{appendix:expts}.}
    \label{fig:EP3_VS_EP4}
\end{figure*}

\section{\EP\,4: algorithm design}\label{sec:highlevel}

In this section, we provide a high-level overview and illustrations to highlight the key components of \EP4 and how it significantly reduces training time.

\subsection{Challenge to Scaling: High Overhead of Projection}

The key development of \EP3 over its contemporaries was that it could train models of this form in $O(p)$ memory. This was a huge improvement over the prior methods which required $O(p^2)$ memory \citep{rudi2017falkon,askotch}.


\begin{table}
\centering
\begin{tabular}{lcccc}
\toprule
\multirow{2}{*}{Algorithm}  
& \multicolumn{3}{c}{FLOPS}                 & \multirow{2}{*}{Memory} \\ 
& \multicolumn{1}{c}{setup} 
& \multicolumn{2}{c}{per sample} 
&
\\ \midrule
\EP{4.0} (ours)               
& \multicolumn{1}{c}{$ O(1)$}      
& \multicolumn{2}{c}{$O(p)$}              
& $O(p)$ 
\\ 
\EP{3.0} 
& $ O(1)$     
& \multicolumn{2}{c}{$O(p^2)$}              
&  $O(p)$      
\\ 
\FALKON            
& \multicolumn{1}{c}{$O(p^3)$}      
& \multicolumn{2}{c}{$O(p)$}              
& $O(p^2)$
\\ \bottomrule
\end{tabular}

\caption{\label{tab:computational complexity}
   \textbf{Algorithm complexity.}  with respect to number of model centers $p$. 
   Here we assumed only a constant number of epochs is needed for large scale experiments. 
   Cost of kernel evaluations and number of classes are assumed to be $O(1)$.
} 
\end{table}

However, \EP\,3 has a high cost $O(mp+p^2)$ per batch of data processed, as summarized in \Cref{tab:computational complexity}. This is especially expensive when $m\ll p$, i.e., when the batch size $m$ is  small compared to the model size $p$.



\subsection{Main Idea: Delayed Projection}



To address the computational complexity challenges, EP4 amortizes projection costs by delaying projections for T iterations . The value of T is a hyperparameter, with an effective selection method detailed in Section~\ref{sec:exact_computational_costs}. Figure~\ref{fig:illustration_ep4_vs_ep3} illustrates this delayed projection mechanism. ( T = 4 in Figure~\ref{fig:illustration_ep4_vs_ep3})


In fact, \EP3 is a special case of \EP4 with parameter $T=1$. However, we show in \cref{eq:optimal_T} that to minimize total time of training, the optimal value of $T$ is in fact proportional to $\frac{p}{m}.$ For this value of $T$, the cost of training per batch is $O(p)$.

Figure~\ref{fig:EP3_VS_EP4} shows how \EP4 and \EP3 perform over training iterations. \EP4 accuracy improves between projections and drops after each projection step. While \EP3 projects at every step, \EP4 maintains comparable accuracy with fewer projections. The left panel of Figure~\ref{fig:EP3_VS_EP4} confirms that both methods reach similar final accuracy, while the right panel shows \EP4 significant speed advantage. With continued training, \EP4 accuracy drops from projections become progressively smaller.

\section{\EP\,4: Algorithm Development and Optimization}\label{sec:derivation}

In this section, we present the EigenPro 4 algorithm in three parts. First, we introduce the algorithm's main components: the pre-projection and projection steps. Then, we detail each of these steps in the following two subsections. Finally, we describe a computational optimization that reduces the runtime of \EP4 by half.


\subsection{Derivation of the \EP4 Algorithm}

\begin{figure*}[h!]
    \centering
    \includegraphics[clip, width=0.9\textwidth]{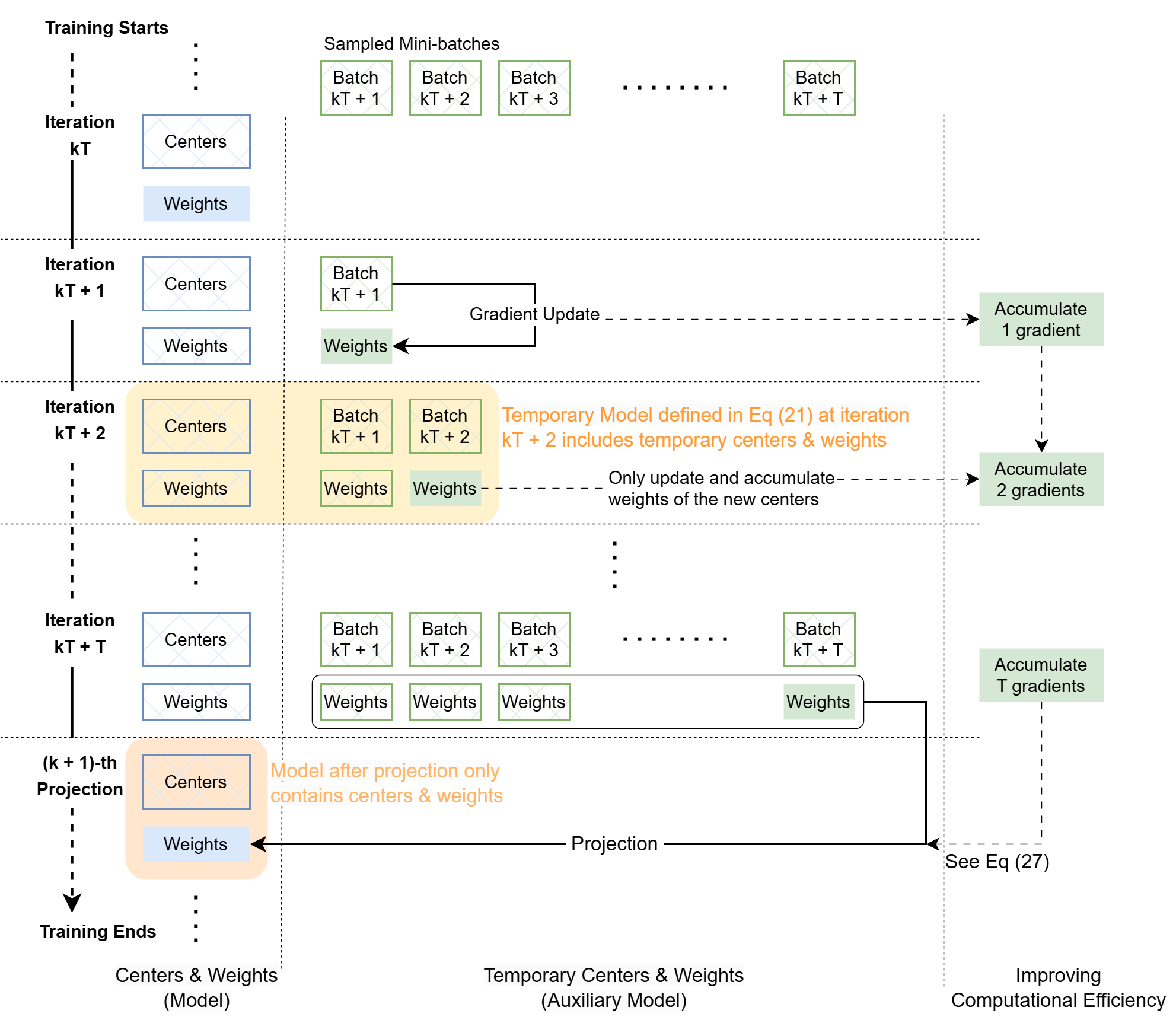}
    \vskip -0.1in
    \caption{\textbf{Overview of iteration scheme for \EP4.} The figure illustrates how the model updates are performed over multiple iterations in EigenPro4. Weights are updated using batches, and gradients are accumulated iteratively until a projection step is executed. Temporary centers and weights are maintained during auxiliary iterations, which are merged through projection after a certain number of batches. This approach reduces the computational cost by accumulating gradients before performing the projection, leading to more efficient batch processing.
    }
    \label{fig:iteration_scheme}
\end{figure*}
 As mentioned previously \(T\) is a crucial hyperparameter that determines the frequency of projection back to \(\Zspan\) after every \(T\) steps. Before the projection step $T$, at every step when a new batch \((X_m, y_m)\) is fetched, it is added to a set defined as the ``temporary centers'' set, denoted by \(Z_{\tmp}\). Starting with an empty set, \(Z_{\tmp} = \emptyset\), we continuously add to \(Z_{\tmp}\) with temporary centers until the step count reaches \(T\).

Formally, the prediction function prior to the projection is no longer fixed and is now expanding. The model can be described as follows:
\begin{align}\label{eq:auxilary_model}
f(x) = \underbrace{\overbrace{\sum_{\z\in Z} \alpha_{\z} K(\x,\z)}^{\text{original model}} + \overbrace{\sum_{\z\in Z_{\tmp}} \beta_{\z} K(\x,\z)}^{\text{temporary model}}}_{\text{auxiliary model}}
\end{align}
where $\alpha_{\z}$ refers to the weights corresponding to the original model center $\z$ and $\beta_{\z}$ refers to the weights corresponding to the temporary model center. We refer to the combination of original model and temporary model  as auxiliary model. The full \EP4 algorithm has been illustrated in \Cref{fig:iteration_scheme} and mathematically can be summerized as following,

\begin{align}\label{eq:ep_plus}
    f_{t} = 
    \begin{cases} 
    \text{proj}_{\Zspan}\left(f_{t-1} - \eta \mathcal{P}^s\{\wt \nabla_f L(f_{t-1})\}\right), &  t\equiv 0 \mod T, \\
    f_{t-1} - \eta \mathcal{P}^s\{\wt \nabla_f L(f_t)\}, & \text{otherwise}.
    \end{cases}
\end{align}
where $\wt\grad L$ is a stochastic gradient of the loss function computed over a mini-batch of data, and $\mc P^s$ is a preconditioner.

\subsection{Pre-projection Steps}
Based on \Cref{eq:ep_plus}, suppose $(X_1,y_1),\ldots, (X_T,y_T)$ are the minibatches of size $m$, and the initial model is $f_0 = K(\cdot,Z)\alphavec$. After $t < T$ step, the following holds,
\begin{align*}
    f_t &= f_{t-1} - \eta\mc P^s K(\cdot,X_t)(f_{t-1}(X_t)-y_t)\\ 
    &= f_0 - \sum_{i=1}^{t}\mc P^s K(\cdot,X_i)(f_{i-1}(X_i)-y_i)
\end{align*}
Replacing $\mc P^s$ with \cref{eq:ps}, $f_0=K(\cdot,Z)\alphavec$ and letting $(\Lambda_{q},\E_{q},\lambda_{q+1})$ be the top-$q$ eigensystem of $K(X_s, X_s)$, we can simplify the update above as following,
\begin{align}\label{eq:update}
    &f_t = K(\cdot,Z)\alphavec_{t} \nonumber
    \\&- \eta \sum_{i=1}^{t} \Bigg( K(\cdot,X_i) \nonumber\\
    &- K(\cdot,X_s) \F \F\tran K(X_s,X_i) \Bigg)(f_{i-1}(X_i) - y_i)\nonumber \\
    &= K(\cdot,Z)\alphavec - \sum_{i=1}^t K(\cdot,X_i) \round{\eta(f_{i-1}(X_i) - y_i)} \nonumber\\
    &\quad + K(\cdot,X_s) \sum_{i=1}^t \eta \F \F\tran K(X_s,X_i) (f_{i-1}(X_i) - y_i)
\end{align}
where, $\D := \Lambda_q\inv\left(\I_q-\lambda_{q+1}\Lambda_q\inv\right)$. This update rule implies that after $t$ steps, the weights corresponding to the original centers $Z$ remain unchanged to $\alphavec$, the weights for the temporary centers $X_i$ are set once to $\eta(f_{i-1}(X_i)-y_i)$ after they are added, and do not change thereafter, and finally, the weights associated with the Nyström samples $X_s$ are updated after each batch via an additive update. This is how we update the weights before projection at step $T$.

\subsection{Projection Step: Update for $t=T$}
Once, we reach step $T$ we need to project $f_T$ into $\Zspan$, or formally 
\begin{align}\label{eq:PGM}
    f_{T} = \text{proj}_{\Zspan}\left(f_{T-1} - \eta \mathcal{P}^s\{\wt\nabla_f L(f_{T-1})\}\right),
\end{align}

Applying Proposition 2 from \citet{abedsoltan_2023}, the solution to this projection problem is as follows,
\begin{align}\label{eq:exact_project_Z}
    f_{T} = K(\cdot,Z)K\inv(Z,Z) f_{T-1}(Z)
\end{align}


Here, we define $\h$ as the \textit{accumulated gradient}.

The final \EP{4} can be found in Algorithm \ref{alg:eigenpro4_real}. Note that we follow the same \textit{inexact projection} scheme used in \citep{abedsoltan_2023} to approximate the exact projection described in \eqref{eq:exact_project_Z}.

The benefit of this approximation is that we don't need to solve the problem exactly in $\Xspan$, nor do we need to project back to $\Zspan$ after each iteration. This approach offers the best of both worlds. In the next section, we demonstrate the effectiveness of this approach compared to prior state-of-the-art methods.

\begin{algorithm}[t]
\caption{\EP4}\label{alg:eigenpro4_real}
\begin{algorithmic}[1]
\REQUIRE Data $(X,\y)$, centers $Z$, batch size $m$, Nystr\"om size $s$, preconditioner level $q,$ projection period $T$ 
\STATE Fetch subsample $X_s\subseteq X$ of size $s$, corresponding weights $\alpha_{s}=0$
\STATE $(\Lambda,\E,\lambda_{q+1})\gets$ top-$q$ eigensystem of $K(X_s,X_s)$ 
\STATE $\D = (\Lambda^{-\!1}\!-\!\lambda_{q\!+\!1}\Lambda^{-\!2})\in\Real^{q\times q}$
\STATE $\F=\E\sqrt{\D} \in \Real^{s\times q}$
\STATE $\M\!=\! K(Z,X_s)\F\in\Real^{p\times q}$
\STATE $Z_\tmp = \emptyset$ , $\alphavec_\tmp=\emptyset$, $\alphavec_{s}=\zero_s$, $\h=\zero_p$
\WHILE{Stopping criterion is not reached}
    \FOR{$t=\{1,2,\ldots,T\}$}
\STATE Fetch minibatch $(X_m,\y_m)$
\STATE $\g_m \gets K(X_m, Z)\alphavec- \y_m$ \label{line:grad1}
\STATE $\g_m\gets\g_m +K(X_m, Z_\tmp)\alphavec_\tmp$ \label{line:grad2}
\STATE $\g_m\gets \g_m+K(X_m, X_s)\alphavec_s$ \label{line:grad3}
\STATE $Z_\tmp$.append($X_m$)
\STATE $\alphavec_\tmp$.append($ -\eta\cdot \g_m$)
\STATE $\alphavec_{s} = \alphavec_{s} +\eta\cdot \F\F\tran K(X_s,X_m)\g_m$ \label{line:update_nys_weights}
\STATE $\h \gets \h- \eta\cdot K(Z,X_m)\g_m$ 
\STATE $\h\gets\h+ \eta\cdot\M\F\tran K(X_s,X_m)\g_m$ \label{line:accumulate}\hfill 
\ENDFOR

\STATE \label{line:projection}$\thetavec\leftarrow\text{proj}_{\Zspan} (\h)$
\STATE $\alphavec\gets\alphavec- \frac{n}m\eta\thetavec$ 
\STATE $Z_\tmp \gets \emptyset$, $\alphavec_\tmp\gets \emptyset$, $\alphavec_s \gets \zero_s$, $\h\gets \zero_p$
\ENDWHILE
\end{algorithmic}
\end{algorithm}

\subsection{Improving Computational Efficiency}
Upon careful examination of the derivations in \eqref{eq:update}, we observe that $f_{T-1}(Z)$ have already been computed previously. This allows us to efficiently reuse $f_{T-1}(Z)$ as follows,
\begin{align}
    &f_{T-1}(Z) = K(Z,Z)\alphavec 
    - \sum_{i=1}^{T-1} K(Z,X_i)\round{ \eta (f_{i-1}(X_i) - y_i) }\nonumber \\
    &\quad + \sum_{i=1}^{T-1} K(Z,X_s) \round{ \eta \F \F\tran K(X_s,X_i) (f_{i-1}(X_i) - y_i) }\nonumber \\
    &= K(Z,Z)\alphavec 
    - \sum_{i=1}^{T-1} K(Z,X_i)\round{ \eta (f_{i-1}(X_i) - y_i) } \nonumber\\
    &\quad + \sum_{i=1}^{T-1} \eta \round{ \M \F\tran K(X_s,X_i) (f_{i-1}(X_i) - y_i) }
\label{eq:intermediate_update}
\end{align}
where $\M = K(Z,X_s)\F$. plugging this in \eqref{eq:exact_project_Z} we obtain, 

\begin{align}
    &f_{T}(Z)= K(\cdot,Z)\left(\alphavec-K^{-1}(Z,Z)\h\right), \nonumber\\
    &\h := -\eta\bigg( \sum_{i=1}^{T-1} 
 K(Z,X_i)\round{   (f_{i-1}(X_i) - y_i) } \nonumber\\
    &\phantom{:=}\quad + \sum_{i=1}^{T-1}  \round{ \M \E_q^\top K(X_s,X_i) (f_{i-1}(X_i) - y_i)} \bigg)
\end{align}

\subsection{Benefits of Approximate Preconditioning}
 Note that the preconditioner $\mc P^s$ allows to drastically improve the speed of the iterations. But at the same time, the Nystr\"om approximation also enables the algorithm to become tractable. Due to this approximate preconditioner, we only need to maintain $s$ temporary centers in the auxiliary model, whereas the exact preconditioner ($s=n$) makes the algorithm intractable. Theoretically we only require $s=O(\log^4 n)$ as shown in \citep{abedsoltan2023nystrom}.

\begin{figure*}[th]
        \centering 
        \begin{subfigure}{1\textwidth}
        \includegraphics[width=1.0\linewidth]{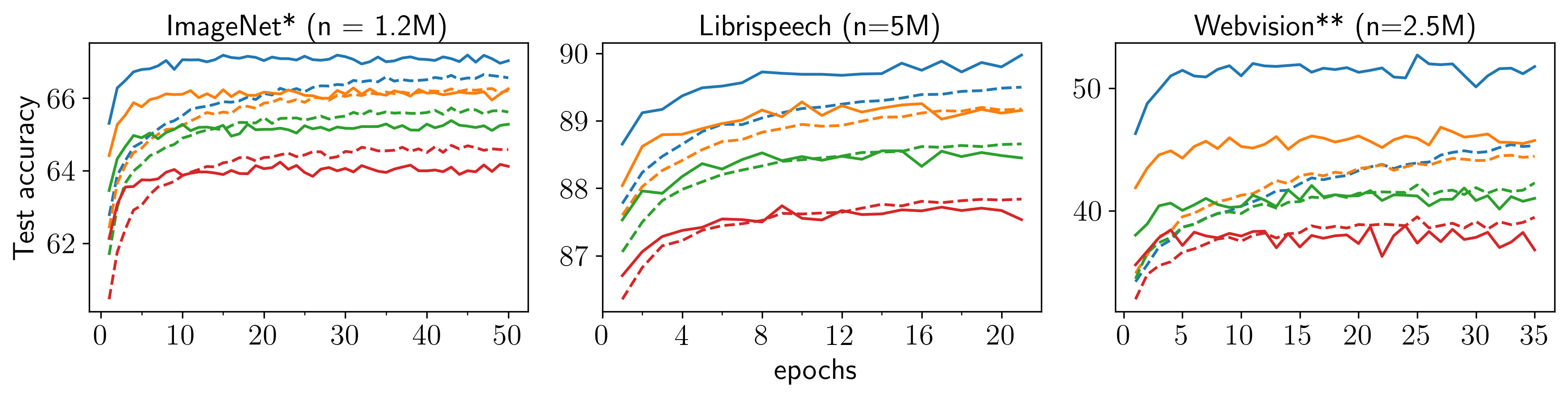}%
        \end{subfigure}
        \begin{subfigure}{1\textwidth}
        \includegraphics[width=1.0\linewidth]{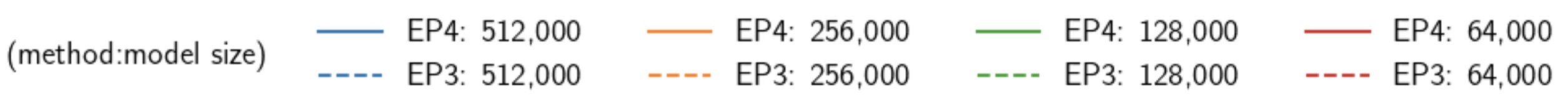}%
        \end{subfigure}
        \caption{ Multi-epoch performance and convergence comparison for \EP3 and \EP4. The detail of the experiment can be found in \Cref{appendix:expts}.
        }
    \label{fig:ep4_ep3_epochs}
\end{figure*}

\section{Numerical experiments}\label{sec:expt}

In this section, we demonstrate that our approach achieves orders of magnitude speedup over the state-of-the-art kernel methods while providing comparable or better generalization performance. We compare the performance of several kernel methods using the following data sets (1) CIFAR5M, CIFAR5M\footnote{\label{mobilenet2}feature extraction using MobileNetV2} \citep{nakkiran2020deep}, (3) ImageNet$^*$, \citep{deng2009imagenet}, (4) Webvision\footnote[2]{feature extraction using ResNet-18}, \citep{li2017webvision}, and (8) Librispeech, \citep{panayotov2015librispeech}. Details about datasets can be found in Appendix \ref{appendix:expts}. While our method is compatible with any kernel function, we opted for the Laplace kernel in our experiments due to its straightforward implementation and strong empirical performance. For handling multi-class classification, we decompose the problem into  separate binary regression tasks, where each class is predicted using targets in $\{0,1\}$. The final classification is determined by selecting the class with the maximum predicted value among all binary predictions.
We run the experiments on a platform with one A100 GPU, one V100 GPU, and one Intel(R) Xeon(R) Gold 6248 CPU.

\paragraph{Substantial Reduction in Per-Epoch Training Time.}
\EP4 has substantially reduced the per-epoch training time, making it the most efficient kernel method on modern machine learning hardware.
In contrast to performing projection every mini-batch iteration as in \EP3, \EP4 schedules one projection every few iterations such that its amortized cost is comparable to that of the standard iterations.
This results in an ideal per-sample complexity $O(p)$, a remarkable improvement over the $O(p^2)$ complexity from \EP3.

In Table \ref{tab:comparison}, we evaluate the performance and computational timing for a single epoch of our proposed model against established kernel regression methods. As noted earlier, \FALKON~ exhibits limitations due to its quadratic memory complexity. For the CIFAR5M$^*$ dataset, training with a model size of 512,000 required 1.3TB of RAM, while scaling to 1M model size necessitated over 5TB. Resource constraints limited our \FALKON~ benchmarks to model sizes of 128,000 and 64,000 for the remaining datasets. While \EP3 addresses these memory constraints, it demonstrates significant computational overhead, particularly evident in the Librispeech dataset where our method, \EP4, achieves a $411\times$ speedup. Notably, \EP4 maintains comparable or superior performance across all evaluated datasets relative to both baseline methods.

\vspace{-15pt}
\begin{figure}[h!]
        \centering 
        \begin{subfigure}{1.0\textwidth}
        \includegraphics[width=0.4\linewidth]{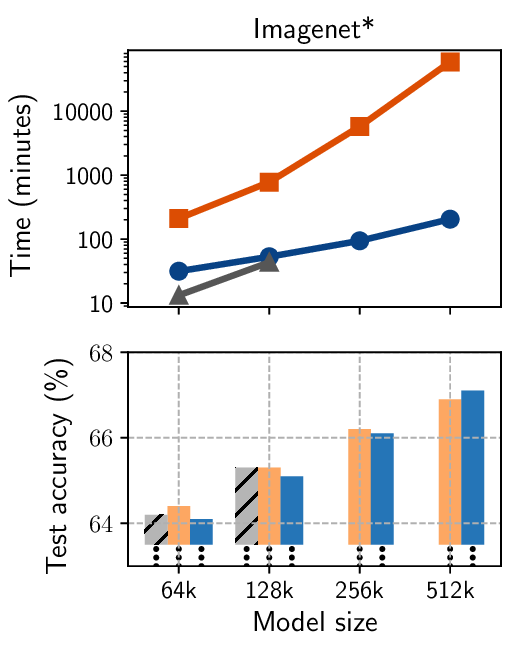}%
        \end{subfigure}
        \vspace{-20pt}
        \begin{subfigure}{1\textwidth}
        \includegraphics[width=0.5\linewidth]{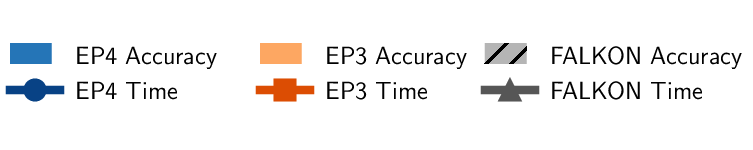}%
        \end{subfigure}
        \caption{ Time and performance comparison for \FALKON, \EP3 and \EP4 for Imagent data set. Details of the experiments can be found in \Cref{appendix:expts}.
        }
    \label{fig:3methods}
\end{figure}

\begin{table*}[th]
\centering
\resizebox{\textwidth}{!}{%
\renewcommand{\arraystretch}{1.2}
\setlength{\tabcolsep}{4pt}
\footnotesize
\begin{tabular}{|c|l|cccc|}
\hline
Model size & Method & CIFAR5M*($n=5$M) & CIFAR5M ($n=6$M) & Librispeech ($n=10$M) & Webvision ($n=5.2$M) \\ \hline
\multirow{3}{*}{\centering p = 64K}
& \EP\,4 & 5m (4.6x, 88\%) & 3m (\textbf{15x}, \textbf{69\%}) & 16m (9.1x, \textbf{86.8\%}) & 2m (\textbf{45.5x}, \textbf{24.3\%}) \\
& \EP\,3 & 23m (1x, \textbf{88.3\%}) & 45m (1x, 68.8\%) & 145m (1x, 85.4\%) & 91m (1x, 24\%) \\
& Falkon & 3m (\textbf{7.67x}, 86.1\%) & 5m (9x, 57.7\%) & 9m (\textbf{16.11x}, 81.0\%) & 4m (22.75x, 21.7\%) \\ \hline
\multirow{3}{*}{\centering p = 128K}
& \EP\,4 & 5m (\textbf{10x}, 88.25\%) & 4m (\textbf{26.25x}, \textbf{70.9\%}) & 19m (\textbf{17.95x}, \textbf{87.8\%}) & 4m (\textbf{49.75x}, \textbf{24.9\%})\\
& \EP\,3 & 50m (1x, \textbf{88.42\%}) & 105m (1x, 70.3\%) & 341m (1x, 84.75\%) & 199m (1x, 24.5\%)\\
& Falkon & 9m (5.56x, 86.55\%) & 11m (9.55x, 59.4\%) & 21m (16.24x, 82.30\%) & 13m (15.31x, 22.4\%) \\ \hline
\multirow{3}{*}{\centering p = 256K}
& \EP\,4 & 7m (\textbf{18.3x}, \textbf{88.61\%}) & 6m (\textbf{130.8x}, \textbf{71.8\%}) & 24m (\textbf{120x}, \textbf{88.33\%}) & 5m (\textbf{106.2x}, \textbf{26\%})\\
& \EP\,3 & 128m (1x, \textbf{88.61\%}) & 785m (1x, 70.53\%) & $\approx$ 2 days (1x) & 531m (1x, 25.52\%) \\
& Falkon & 38m (3.37x, 86.73\%) & OOM & OOM & OOM \\ \hline
\multirow{3}{*}{\centering p = 512K}
& \EP\,4 & 12m (\textbf{44.25x}, \textbf{88.58\%}) & 10m ($>$ \textbf{288x}, 72.9\%) & 36m ($>$ \textbf{200x}, \textbf{88.89\%}) & 11m (\textbf{240x}, \textbf{27.3\%})\\
& \EP\,3 & 531m (1x, 88.56\%) & $>$ 2 days (1x) & $>$ 5 days (1x) & 2 days (1x) \\
& Falkon & 240m (2.21x, 86.71\%) & OOM & OOM & OOM  \\ \hline
\multirow{3}{*}{\centering p = 1M}
& \EP\,4 & 21m ($>$ \textbf{274x}, \textbf{88.7\%}) & 17m ($>$ \textbf{508x}, \textbf{73.8\%}) & 70m ($>$ \textbf{411x}, \textbf{89.5\%}) & 21m ($>$ \textbf{686x}, \textbf{29.3\%}) \\
& \EP\,3 & $>$ 4 days (1x) & $>$ 6 days (1x) & $>$20 days (1x) & $>$ 10 days (1x) \\
& Falkon & \textsf{OOM} & \textsf{OOM} & \textsf{OOM} & \textsf{OOM} \\ \hline
\end{tabular}
}
\caption{Comparison of \EP4, \EP3, and \FALKON~ across different model sizes and datasets after $1$ epoch. The values in parentheses represent the speedup in time over \EP3 and the resulting accuracy. Details of the experiments can be found in \Cref{appendix:expts}.\\
}
\label{tab:comparison}
\end{table*}


\paragraph{Linear Scaling with Model Size.}
The total training time and memory usage of \EP4 scales linearly with the model size.
In comparison, the time required for a single \EP3 iteration grows quadratically with the model size, while the preprocessing time for \FALKON~ grows cubically.
Furthermore, the memory demand of \FALKON~ increases quadratically with the model size.
In practice, we are unable to run it with large model sizes, e.g., 128,000 centers for ImageNet data.

We summarize all empirical results in Figure \ref{fig:3methods} and demonstrate that our method achieves both linear memory complexity and linear time complexity (empirically verified) with respect to model size, offering the best of both worlds. For the ImageNet dataset, we trained all methods until convergence.
While \EP3 does not have the quadratic memory scaling problem, Figure \ref{fig:3methods} shows that even for a relatively small dataset like ImageNet with 1M data points, training a model size of 512,000 centers requires approximately 43 days on a single GPU to reach convergence (about 100 epochs). In contrast, our proposed model achieves convergence in approximately 3 hours, requiring only 15 epochs, with each epoch being significantly more computationally efficient than \EP3 (see Table \ref{tab:comparison}).

\paragraph{Faster Convergence with \EP4.}
\EP4 generally demonstrates the fastest convergence among all tested methods. In certain cases, such as ImageNet with 1.2 million model centers, \EP4 converges in less than $10\%$ of the epochs needed by other methods, while also delivering superior model performance.
Figure \ref{fig:ep4_ep3_epochs} compares \EP4 and \EP3 across multiple training epochs, following the experimental setup established in \citet{abedsoltan_2023}. Despite \EP4's linear time complexity per iteration (compared to \EP3's quadratic complexity), it demonstrates faster convergence with fewer epochs. This efficiency gain is particularly pronounced for larger model sizes, where \EP4 maintains or exceeds \EP3's accuracy while requiring significantly fewer epochs. These results empirically validate that \EP4's algorithmic improvements translate to practical benefits: not only does each iteration run faster, but fewer iterations are needed to achieve optimal performance across diverse datasets. This empirically shows that our model has a linear time complexity with respect to the model size.

\section{conclusion}
In this work, we introduced EigenPro 4, an advancement in training large kernel models that achieves linear time complexity per iteration and linear memory scaling with model size. By implementing a delayed projection strategy, we addressed the high computational overhead previously associated with frequent projections, achieving significant time and memory efficiency improvements over EigenPro 3 and Falkon. Our empirical results on diverse datasets highlight EigenPro 4 ability to match or exceed the performance of prior methods with vastly reduced computational resources. Specifically, the algorithm demonstrates both faster convergence and superior scalability, enabling training with model sizes and datasets that were previously infeasible due to memory and time constraints.

Furthermore, EigenPro 4 design opens up new possibilities for parallelization, as it is well-suited for multi-GPU and distributed architectures. Future work will explore these aspects, further expanding its potential in real-world applications requiring efficient, scalable kernel methods for massive data volumes.

\paragraph{Acknowledgements:} 

We acknowledge support from the National Science Foundation (NSF) and the Simons Foundation for the Collaboration on the Theoretical Foundations of Deep Learning through awards DMS-2031883 and \#814639,  the  TILOS institute (NSF CCF-2112665), and the Office of Naval Research (N8644-NV-ONR). 
This work used 
ACCESS (Advanced cyberinfrastructure coordination ecosystem: services \& support) which is supported by NSF grants numbers \#2138259, \#2138286, \#2138307, \#2137603, and \#2138296. Specifically, we used the resources from SDSC Expanse GPU compute nodes, and NCSA Delta system, via allocations TG-CIS220009. This work was done in part while AA was visiting the Simons Institute for the Theory of Computing. PP was supported by the DST INSPIRE Faculty Fellowship, and a Thakur Family Chair at C-MInDS, IIT Bombay.

\bibliography{aux/ref}
\bibliographystyle{plainnat}
\onecolumn
\appendix

\section{Convergence analysis}\label{Appendix:a}

In this section, we derive EigenPro4.0-Exact (Algorithm \ref{alg:eigenpro4_exact}), a precursor to EigenPro4.0. However, this version does not scale efficiently. In Section \ref{sec:derivation}, we enhance its scalability by introducing stochastic approximations, resulting in EigenPro4.0 (Algorithm \ref{alg:eigenpro4_real}).

Recall that the derivatives of the loss function, as defined in \eqref{eq:ep3}, lie in the span of the training data, denoted as $\Xspan$. However, these derivatives cannot directly update the model, which resides in the span of the model centers, $\Zspan$. To address this, we first fit the labels within the $\Xspan$ and then project the solution into the $\Zspan$ . This process is repeated iteratively on the residual labels until convergence, as outlined in Algorithm \cref{alg:eigenpro4_exact}.

\begin{algorithm}[th]
\caption{\EP{4}-Exact}\label{alg:eigenpro4_exact}
\begin{algorithmic}[1]
\REQUIRE Data $(X,\y)$, centers $Z$
\STATE $\bm{\Tilde{y}}_0 = \y$
\FOR{$t = 1, 2, \ldots$}
    \STATE $\alphavec_t = K\inv(X,X)\bm{\Tilde{y}}_t$
    \STATE $K(\cdot,Z)\betavec_t = \text{proj}_{\Zspan} \left(K(\cdot,X)\alphavec_t\right)$ 
    \STATE $\bm{\Tilde{y}}_{t+1} = \y - K(X,Z)\betavec_t$
\ENDFOR
\end{algorithmic}
\end{algorithm}

The following proposition provides the fixed point analysis for this algorithm.

\begin{proposition}
Consider any dataset $X, \y$ and a choice of model centers $Z$, with a kernel function $K: \mathbb{R}^d \times \mathbb{R}^d \to \mathbb{R}$. Assume that $K(X,X)$ and $K(Z,X)$ are full Rank. Then, Algorithm \ref{alg:eigenpro4_exact} converges to the following solution:

\begin{equation}
    \hat{f} = K(\cdot,Z)\left(K(Z,X)K^{-1}(X,X)K(X,Z)\right)^{-1}K(Z,X)K^{-1}(X,X) \y.
\end{equation}
Furthermore, if $\y = K(X,Z) \betavec^* + \xivec$, where $\xivec$ is a vector of independent centered random noise with $\mathbb{E}[\xi_i^2] = \sigma^2$, then

\begin{align*}
    \lim_{t \to \infty} \mathbb{E}[\betavec_t] = \betavec^*, \quad &\lim_{t \to \infty} \frac{\mathbb{E}[\|\betavec_t - \betavec^*\|^2]}{\sigma^2} = \\&{\rm tr}\round{\round{K(Z,X)K^{-1}(X,X)K(X,Z)}^{-2}K(Z,X)K^{-2}(X,X)K(X,Z)}.
\end{align*}

\end{proposition}

This algorithm has a major drawback, as solving the problem in the $\Xspan$ is inherently more challenging. However, in the next section, we demonstrate how to effectively scale this approach.



\begin{proposition}
Consider any dataset $X, \y$ and a choice of model centers $Z$, with a kernel function $K: \mathbb{R}^d \times \mathbb{R}^d \to \mathbb{R}$. Assume that $K(X,X)$ and $K(Z,X)$ are full Rank. Then, Algorithm \ref{alg:eigenpro4_exact} converges to the following solution:

\begin{equation}
    \hat{f} = K(\cdot,Z) K^{+}(Z,X) \y.
\end{equation}
Furthermore, if $\y = K(X,Z) \betavec^* + \xivec$, where $\xivec$ is a vector of independent centered random noise with $\mathbb{E}[\xi_i^2] = \sigma^2$, then

\begin{align*}
    \lim_{t \to \infty} \mathbb{E}[\betavec_t] = \betavec^*, \quad &\lim_{t \to \infty} \frac{\mathbb{E}[\|\betavec_t - \betavec^*\|^2]}{\sigma^2} = \\&{\rm tr}\round{\round{K(Z,X)K^{-1}(X,X)K(X,Z)}^{-2}K(Z,X)K^{-2}(X,X)K(X,Z)}.
\end{align*}

\end{proposition}

\begin{proof}
We begin by expressing Algorithm \ref{alg:eigenpro4_exact} recursively and substituting $\text{proj}_{\Zspan}$ with the expression in \eqref{eq:projection}. Recall that $f_t = K(\cdot, Z)\betavec_t$ with base case $\betavec_0 = 0$. The update rule for $\betavec_t$ is given by:

\begin{equation}
    \betavec_t = K^{-1}(Z,Z)K(Z,X)K^{-1}(X,X)(\y - K(X,Z)\betavec_{t-1}) + \betavec_{t-1}.
\end{equation}

Let us define the matrices:
\[
B := K^{-1}(Z,Z)K(Z,X)K^{-1}(X,X), \quad C := BK(X,Z) - I,
\]
which allows us to rewrite the recursion more succinctly:

\begin{equation}
    \begin{aligned}
        \betavec_t &= B(\y - K(X,Z)\betavec_{t-1}) + \betavec_{t-1} \\
        &= B\y - C\betavec_{t-1} = B\y - CB\y + C^2\betavec_{t-2} \\
        &\vdots \\
        &= \left(\sum_{i=0}^{t-1} (-1)^i C^i \right) B\y.
    \end{aligned}
\end{equation}

As the number of iterations tends to infinity, we can define the infinite series sum:
\[
S := \sum_{i=0}^{\infty} (-1)^i C^i.
\]
Observe that:
\[
S + CS = I.
\]
Substituting the definition $C = BK(X,Z) - I$ and $B = K^{-1}(Z,Z)K(Z,X)K^{-1}(X,X)$, we have:
\[
K^{-1}(Z,Z)K(Z,X)K^{-1}(X,X)K(X,Z)S = I.
\]
Thus, this simplifies to:
\[
S = \left(K(Z,X)K^{-1}(X,X)K(X,Z)\right)^{-1}K(Z,Z).
\]

Therefore, the final solution converges to:

\begin{equation}
    \hat{f} = K(\cdot,Z) \left(K(Z,X)K^{-1}(X,X)K(X,Z)\right)^{-1}K(Z,X)K^{-1}(X,X)\y.
\end{equation}

Substituting $\y = K(X,Z) \betavec^* + \xivec$ readily completes the second claim.













\end{proof}

\section{Computational complexity comparison}
\label{sec:exact_computational_costs}
We assume that EigenPro4 is processing $T$ batches of data at once before running the post-processing step of projection. Here we show we calculated the optimal value of $T$.

\paragraph{Cost for processing $t^{\rm th}$ batch of data.}
For a some $k\in\Natural$, let $kT < t \leq (k+1)T$. See \Cref{tab:cost_breakdown}.
\begin{table}
    \centering
    \begin{tabular}{lll}
    \toprule
    line & computation & flops \\
    \midrule
    \ref{line:grad1} 
    &$K(X_m,Z)\alphavec - \y_t$ &$mp$ 
         \\
    \ref{line:grad2} &
    $K(X_{m},Z_{\tmp})\alphavec_\tmp$&
    $m^2(t-kT-1)$
           \\
    \ref{line:grad3} &
    $K(X_m,X_s)\alphavec_s$ &
    $ms$ 
         \\
    \ref{line:update_nys_weights},\ref{line:accumulate} & $\h_1 :=\F\tran K(X_s,X_m)\g_m\in\Real^q$ &
    $ms+sq$ \\
    \ref{line:update_nys_weights} &
    $\F \h_1$ & $sq$\\
    \ref{line:accumulate} & 
    $K(Z,X_m)\g_m$ & $mp$\\
    \ref{line:accumulate}& $\M\h_1$ & $pq$\\
    \bottomrule
    \end{tabular}
    \caption{\label{tab:cost_breakdown}Computational cost analysis of \Cref{alg:eigenpro4_real} for processing batch $t$ for $kT < t \leq (k+1)T$ for some $k\in\Natural$.\\
    The cost of processing batch $t$ without the post-processing adds up to $2mp+2ms+2sq+pq+m^2(t-kT-1)$ flops.
    }
\end{table}

\paragraph{Cost of processing $T$ batches of data before post-processing}
The total cost for processing $T$ batches $t=kT+1$ to $t=(k+1)T$ before the projection is the sum of the above
\begin{align}
    T(2mp+2ms+2sq+pq)+m^2\sum_{t=kT+1}^{(k+1)T}(t-kT-1)
    =T(2mp+2ms+2sq+pq)+m^2\frac{T(T-1)}2
\end{align}

\paragraph{Average of processing $T$ batches of data with post-processing}
Assuming the post processing involves $T_{\sf ep2}$ epochs of \EP2, the average cost of processing $T$ batches is
\begin{align}
    \label{eq:optimal_T}\frac{T(2mp+2ms+2sq+pq)+m^2\frac{T(T-1)}2 + p^2 T_{\sf ep2}}{T}
\end{align}
A simple calculation shows that
\begin{align}
    T^\star = \frac{p}{m}\sqrt{2T_{\sf ep2}}
\end{align}
minimizes the average time above. The average cost of processing a batch is thus
\begin{align}
    2mp (1+\sqrt{2T_{\sf ep2}})+2ms+2sq+pq
\end{align}

\begin{table}[h]
\centering
\begin{tabular}{ccccc}
\toprule
\multirow{2}{*}{Algorithm}  
& \multicolumn{3}{c}{FLOPS}                 & \multirow{2}{*}{Memory} \\ 
& \multicolumn{1}{c}{setup} 
& \multicolumn{2}{c}{per batch$^{*}$} 
&
\\ \midrule
\EP{4.0}                
& \multicolumn{1}{c}{$O(s^2q)$}      
& \multicolumn{2}{r}{$2mp (1+\sqrt{2T_{\sf ep2}})+2ms+2sq+pq$}              
& $s^2+p(1+\sqrt{2 T_{\sf ep2}})$ 
\\ 
\EP{3.0} 
& \multicolumn{1}{c}{${O(s^2q)}$}      
& \multicolumn{2}{r}{$2mp + p^2 T_{\sf ep2} + 2ms+2sq+pq$}              
&  $s^2+p$      
\\ 
\FALKON                  
& \multicolumn{1}{c}{$O(p^3)$}      
& \multicolumn{2}{l}{$2mp$}              
& $p^2$
\\ \bottomrule
\end{tabular}

\caption{\label{tab:computational complexity comparison}
   \textbf{Comparing complexity of algorithms.}  Number of training samples $n$, number of model centers $p$, batch size $m$, Nystr\"om sub-sample size $s$, preconditioner level $q$.
    Here we assumed only a constant number of epochs of \EP{2.0} is needed for large scale experiments. Cost of kernel evaluations and number of classes are assumed to be $O(1)$, also it is reasonable to assume $p\gg s\gg q$.\\
    $^*$ FLOPS per iteration reported are amortized over multiple batches processed.
} 

\end{table}

\section{Experiments Results}\label{appendix:expts}

\subsection{Computational resources used}
This work used the Extreme Science and Engineering Discovery Environment (XSEDE) \citep{XSEDE}. We used machines with NVIDIA-V100, NVIDIA-A100 and NVIDIA-A40 GPUs, with a V-RAM up to 1.3 T, and 8x cores of Intel(R) Xeon(R) Gold 6248 CPU @ 2.50GHz with a RAM of 100 GB. NOte that we had 1.3T of ram for just one experiment CIFAR5M$^*$, for the rest of expermients we where constraint with 400G of ram.

\subsection{Datasets}

We perform experiments on these datasets: (1) CIFAR10, 
\citet{krizhevsky2009learning}, (2) CIFAR5M,  
\citet{nakkiran2020deep}, (3) ImageNet, 
 (4) Webvision.\citet{li2017webvision}, and (5) librispeech.

\paragraph{CIFAR5M.} In our experiments, we utilized both raw and embedded features from the CIFAR5M data-set. The embedded features were extracted using a MobileNetv2 model pre-trained on the ImageNet data-set, obtained from \textit{timm} library \citet{rw2019timm}. We indicate in our results when pre-trained features were used by adding an asterisk (*) to the corresponding entries.

\paragraph{ImageNet.} In our experiments, we utilized embedded features from the ImageNet data-set. The embedded features were extracted using a MobileNetv2 model pre-trained on the ImageNet dataset, obtained from \textit{timm} library \citet{rw2019timm}. We indicate in our results when pre-trained features were used by adding an asterisk (*) to the corresponding entries.

\paragraph{Webvision.} In our experiments, we utilized embedded features from the Webvision data-set. The embedded features were extracted using a ResNet-18 model pre-trained on the ImageNet dataset, obtained from \textit{timm} library \citet{rw2019timm}. Webvision data set contains 16M images in 5k classes. However, we only considered the first 2k classes.

\paragraph{Librispeech.}Librispeech \citet{panayotov2015librispeech} is a large-scale (1000 hours in total) corpus of 16 kHz English speech derived from audio books. We choose the subset train-clean-100 and train-clean-300 (5M samples) as our training data, test-clean as our test set. The features are got by passing through a well-trained acoustic model (a VGG+BLSTM architecture in \citet{hui2020evaluation} ) to align the length of audio and text. It is doing a 301-wise classification task where different class represents different uni-gram \citet{jurafsky2000speech}. The implementation of extracting features is based on the ESPnet toolkit \citet{watanabe2018espnet}.

\subsection{Experiments details}

\paragraph{\Cref{fig:time_comparision}} This experiment used CIFAR5M$^*$ data set, where embedding has been generated using a pre-trained mobile-net network mentioned earlier. this is the only experiment that we had access to 1.3T of VRAM. We set the bandwidth to $5.0$ and use $1k$ Nystrom samples with preconditioning level of size $100$. We used float16 for this experiment.

\paragraph{\Cref{fig:EP3_VS_EP4}} This experiment has been run over Webvision data set with extracted embedding through Resnet18. the model size here is set to $100k$ number of centers. The bandwidth used is $5.0$, $1k$ Nystrom samples with preconditioning level of size $100$. We used float16 for this experiment.

\paragraph{\Cref{fig:ep4_ep3_epochs}} We follow the setting in \citet{abedsoltan_2023}. The bandwith used here is 20 for Librispeach and Webvision and 16 for imagnet. Here again we used extracted feature of these datasets mentioned earlier. The precision used here is float32. with $10k$ Nystrom samples with preconditioning level of size $1000$.

\paragraph{\Cref{fig:3methods}} We used the same experiment setting in \Cref{fig:ep4_ep3_epochs}.

\paragraph{\Cref{tab:comparison}} For all datasets here we used bandwidth of 5.0 with $1k$ Nystrom samples with preconditioning level of size $100$. We used float16 for all dataset except for Librispeach where we used float32. Further, we note that \FALKON~ latest library ran out of GPU memory for model sizes larger than 256000 number of centers that is the reason we could not run it for 256000. And as mentioned for model sizes 521000 and above the algorithm has inherent quadratic scaling with respect to model size and we ran out of VRAM. In the plot we refer to both of these memry issues as OOM.

\end{document}